\documentclass[letterpaper, 10 pt, conference]{ieeeconf}  

\IEEEoverridecommandlockouts

\usepackage{upgreek}
\usepackage{graphicx}
\usepackage{subfig}
\usepackage{amsmath}
\usepackage{amsthm} 
\usepackage{amssymb}
\usepackage[font=footnotesize]{caption} 
\usepackage{subfig} 
\usepackage[noadjust]{cite} 
\usepackage{color}
\usepackage{algorithm} 
\usepackage{algpseudocode} 
\usepackage{enumerate}
\usepackage[hidelinks,colorlinks=false]{hyperref}
\usepackage[left=0.75in, right=0.75in, top=0.75in, bottom=0.75in]{geometry}
\usepackage{authblk} 
\usepackage{arydshln} 
\usepackage{subfloat}
\usepackage{mdframed}
\usepackage{bm}
\usepackage{tikz}
\usetikzlibrary{calc} 
\usetikzlibrary{shapes} 
\usetikzlibrary{chains}
\usetikzlibrary{fit}
\usetikzlibrary{arrows}
\usetikzlibrary{decorations.text} 
\usetikzlibrary{decorations.markings}
\usetikzlibrary{decorations.pathmorphing} 
\usetikzlibrary{shadows}
\usetikzlibrary{patterns}
\usetikzlibrary{matrix}
\usepackage{pgfplots}
\usepackage[europeanresistors]{circuitikz}
\usepackage[outline]{contour} 
\contourlength{1.5pt}
\newtheorem{theorem}{Theorem}
\newtheorem{lemma}{Lemma}

\newcommand{\I}{\mathbf{I}}

\newcommand{\A}{\mathbf{A}}

\newcommand{\g}{\mathbf{g}}

\renewcommand{\H}{\mathbf{H}}
\newcommand{\h}{\mathbf{h}}

\newcommand{\M}{\mathbf{M}}
\newcommand{\N}{\mathcal{N}}

\renewcommand{\P}{\mathbf{P}}
\newcommand{\p}{\mathbf{p}}
\newcommand{\R}{\mathbf{R}}
\renewcommand{\S}{\mathbf{S}}

\newcommand{\vel}{\mathbf{v}}

\newcommand{\w}{\mathbf{w}}

\newcommand{\x}{\mathbf{x}}

\newcommand{\z}{\mathbf{z}}

\newcommand{\zero}{\mathbf{0}}

\graphicspath{{figures/}}

\title{A Cooperative Bearing-Rate Approach for \\
Observability-Enhanced Target Motion Estimation
}
\author{Canlun Zheng$^{1,2}$, Hanqing Guo$^2$, Shiyu Zhao$^2$
\thanks{$^1$College of Computer Science and Technology, Zhejiang University, Hangzhou, China. $^2$WINDY Lab, Department of Artificial Intelligence, Westlake University, Hangzhou, China. {\{zhengcanlun, guohanqing, zhaoshiyu\}@westlake.edu.cn}
This research work was supported by National Natural Science Foundation of China (Grant No. 62473320). Corresponding author: S. Zhao.
}
}
\begin{document}
\maketitle
\begin{abstract}	
Vision-based target motion estimation is a fundamental problem in many robotic tasks. The existing methods have the limitation of low observability and, hence, face challenges in tracking highly maneuverable targets. Motivated by the aerial target pursuit task where a target may maneuver in 3D space, this paper studies how to further enhance observability by incorporating the \emph{bearing rate} information that has not been well explored in the literature. The main contribution of this paper is to propose a new cooperative estimator called STT-R (Spatial-Temporal Triangulation with bearing Rate), which is designed under the framework of distributed recursive least squares. This theoretical result is further verified by numerical simulation and real-world experiments. It is shown that the proposed STT-R algorithm can effectively generate more accurate estimations and effectively reduce the lag in velocity estimation, enabling tracking of more maneuverable targets.
\end{abstract}

\section{Introduction}

The research work in this paper is motivated by the task of vision-based aerial target pursuit, where pursuer micro aerial vehicles (MAVs) can detect, localize, and follow a target MAV \cite{vrba2020marker,li2022three,saini2019markerless}. This task is inspired by bird-capturing-bird behaviors in nature \cite{brighton2019hawks}. 
One critical problem in countering malicious MAVs is vision-based target motion estimation, which is also a fundamental problem in many other tasks such as autonomous driving and obstacle avoidance \cite{griffin2021depth}.

The most common method for vision-based target motion estimation formulates vision as a \emph{bearing-only} sensing approach \cite{li2022three, su2022bearing, sharma2011graph}.
To recover the distance information in the bearing-only case, the observability condition that the observer should have a higher-order motion than the target must be satisfied \cite{ning2024bearingangle}. Our previous works have studied how the observer should maneuver to enhance observability \cite{li2022three, flayac2023nonuniform,ning2024bearingangle}.
One effective method to enhance observability is to use multiple observers to estimate cooperatively \cite{sharma2011graph, sharma2013bearing, zhao2014optimal,schiano2018dynamic}.
In our latest work \cite{zheng2023optimal}, we introduce a cooperative estimator that can efficiently localize a target using the spatial-temporal information derived from bearing measurements.

Although multiple observers can triangulate to determine the target's position, the target's velocity cannot be directly calculated and can only be predicted through filtering \cite{shen2013vision, gardner2022pose, riaz2022state}. Such velocity predictions may exhibit considerable lag when the target maneuvers. 

To address this problem, we explore the \emph{bearing rate measurement} (i.e., the derivative of the bearing vector) that contains valuable \emph{velocity information} of the target. 
Bearing rate measurements can be obtained from vision sensing, especially optical flow \cite{2020RAFT}.

Currently, the bearing rate information in motion estimation problems has only been studied in the case of a single observer. The bearing rate information has also been studied in the context of cooperative formation control problems to improve control performance \cite{zhao2019bearing}, but it is about formation control instead of motion estimation.

\begin{figure}
\centering
\includegraphics[width=1\linewidth]{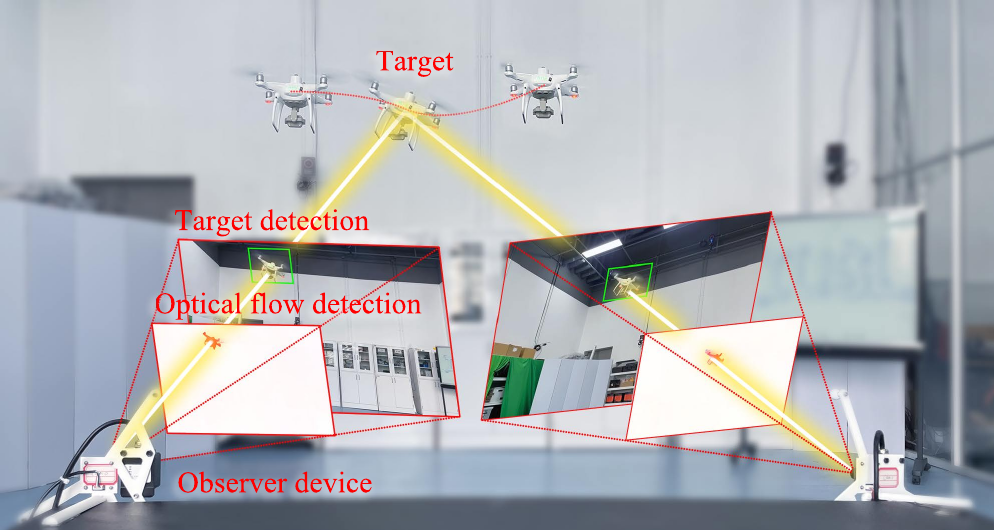}
\caption{Two observers cooperatively estimate a target MAV state with target detection and optical flow detection.}
\label{fig_env}
\end{figure}

The contributions of this paper are summarized as follows.

1) To obtain the bearing rate measurement from vision, we jointly employ the target detector YOLOv5 and the deep-learning optical flow detector RAFT \cite{2020RAFT}. 
By focusing on the central pixel of the target's bounding box and the external parameters, we can obtain a stable bearing rate measurement in the world frame.

2) We establish the pseudo-linear measurement equation of the bearing rate measurements and fuse it with other information under the distributed recursive least squares \cite{mateos2012distributed}. In particular, we extend the spatial-temporal triangulation (STT) method, proposed in \cite{zheng2023optimal}, into a new cooperative estimator called STT-R (Spatial-Temporal Triangulation with bearing Rate).

3) Necessary and sufficient observability conditions are presented to theoretically reveal the role of the additional bearing rate measurement in the estimation process. When there are multiple observers, it is remarkable that the system is observable for any single time step due to the introduction of multiple observers and the additional bearing rate information.

Finally, both numerical simulations and real-world experiments have been conducted to verify the effectiveness of the proposed approach. 
\begin{figure*}
\centering
\includegraphics[width=1\linewidth,trim= 0 85 0 75,clip]{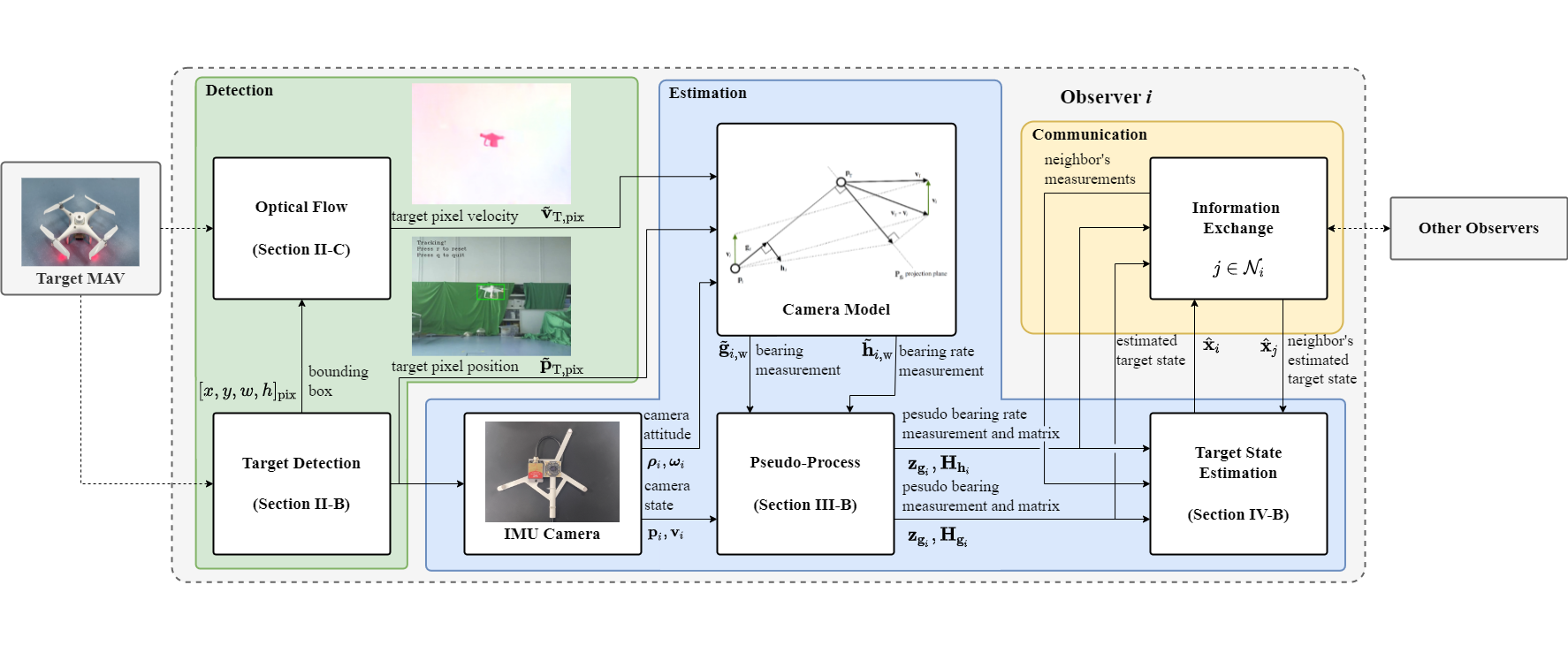}
\caption{System architecture of the proposed target estimation approach. The system includes three parts: detection, estimation, and communication.}
\label{fig_framework}
\end{figure*}

\section{Problem Setup}\label{sec_prob}

\subsection{Problem Statement}\label{subsec_prob_state}

Consider a scenario where multiple cameras observe a moving target in 3D space.
Let $\mathbf{p}_i, \vel_i\in\mathbb{R}^3$ $(i = 1, 2, ..., n)$ be the position and velocity of observer $i$.
Let $\p_{{\rm T}}, \vel_{{\rm T}} \in\mathbb{R}^3$ be the true position and velocity of the target, respectively. $\p_{{\rm T}}$ and $\vel_{{\rm T}}$ are unknown and need to be estimated. The estimation relies on two key quantities: bearing and bearing rate.

First, the bearing vector between observer $i$ and the target is
\begin{align}
\g_i = \frac{\mathbf{p}_{{\rm T}} - \mathbf{p}_i}{r_i}\in\mathbb{R}^3,
\label{eq_bearingAccurate}
\end{align}
where $r_i = \|\p_{{\rm T}} - \mathbf{p}_i\|_2$.
Second, the bearing rate $\h_i$ can be obtained by taking the time derivative on both sides of \eqref{eq_bearingAccurate}:
\begin{align}
	\h_i
 = \frac{(\mathbf{v}_{{\rm T}} - \mathbf{v}_i) - \g_i\dot{r}_i}{r_i}, \label{eq_gdot_rdot}
\end{align}
where $\dot{r_i}$ is unknown. We next calculate $\dot{r}_i$. Multiplying $\g_i^Tr_i$ on both sides of \eqref{eq_bearingAccurate} yields 
\begin{align}
    r_i \g_i^T\g_i = \g_i^T(\p_{\rm T} - \p_i).\label{eq_r_i}
\end{align}
Hence $r_i =\g_i^T(\p_{\rm T} - \p_i)$.
Taking the derivative of both sides of \eqref{eq_r_i} gives
\begin{align}
	\dot{r}_i & = \h^T_i(\mathbf{p}_{{\rm T}} - \mathbf{p}_i) + \g^T_i(\vel_{\rm T} - \vel_i) 
= \g^T_i(\vel_{\rm T} - \vel_i). \label{eq_dot_r}
\end{align}
Finally, substituting \eqref{eq_dot_r} into \eqref{eq_gdot_rdot} we have
\begin{align}
\h_i& = \frac{\P_{\g_i} (\vel_{\rm T} - \vel_i)}{r_i}, \label{eq_gdot}
\end{align}
where \(\h_i^T\g_i=0\) since  $\h_i$ is orthogonal to $\g_i$,  $\P_{\g_i} = \I - \g_i\g^T_i$ \cite{zhao2019bearing}. 

The problem to be solved in this paper is stated as follows. Suppose every observer 1) knows its own position $\p_i$ and velocity $\vel_i$; 2) can obtain a noisy measurement of $\g_i$ and a noisy measurement of $\h_i$; and 3) can receive specific information from its nearest neighbors (the measurements and estimated target's state). The goal is to estimate the target's position and velocity, $\p_{{\rm T}}$ and $\vel_{{\rm T}}$, accurately and promptly.

\subsection{Bearing Measurement}\label{sub_g}

Many deep-learning detectors are available for target detection. For example, the YOLO series is a popular one-stage object detector based on convolutional neural networks \cite{redmon2018yolov3, bochkovskiy2020yolov4}. It can achieve a good balance between accuracy and speed \cite{2021Air, wu2020using, isaac2021unmanned}.

A detector can generate a bounding box that tightly surrounds the target in the image.
Suppose $[\tilde{\p}_{i,{\rm T}}(x),\tilde{\p}_{i,{\rm T}}(y)]^T_{i, {\rm pix}}$ is the pixel coordinate of the center of the bounding box in the image of observer $i$.
The noisy measurements of $\g_i$ in the camera frame can be calculated as
\begin{align}
\tilde{\g}_{i, {\rm c}} & =	\mathbf{K}^{\rm c}_{i, {\rm pix}}\begin{bmatrix}
\tilde{\p}_{i,{\rm T}}(x)\\ \tilde{\p}_{i,{\rm T}}(y) \\ 1
\end{bmatrix} = \g_{i, {\rm c}} + \mathbf{w}_{\g_{i, {\rm c}}},
\end{align}
where $\mathbf{w}_{\g_{i,c}}$ is noise, $\mathbf{K}^{\rm c}_{i, {\rm pix}}$ is a static transformation matrix from the pixel frame to the camera frame. It can be obtained from the camera's intrinsic parameters.
Then, the bearing measurement expressed in the world frame can be obtained as
\begin{align}\label{eq_rel_cam2world_g}
\tilde{\g}_i &= \R_{i, {\rm c}}^{\rm w}\tilde{\g}_{i, {\rm c}}= \g_i +\R_{i, {\rm c}}^{\rm w}\mathbf{w}_{\g_{i, {\rm c}}},
\end{align}
where $\R_{i, {\rm c}}^{\rm w}$ is the rotation from the camera frame to the world frame. It can be calculated based on the attitude of the $i$th observer's camera.

\subsection{Bearing Rate Measurement}\label{sub_h}

Unlike bearing measurements, bearing rate measurements have not been widely used for motion estimation. In fact, we can use the deep-learning optical flow model RAFT \cite{2020RAFT} trained on synthetic datasets \cite{raft-things} to extract the velocity of every pixel.
With the bounding box provided by a target detector such as YOLO, we can obtain the velocity of the target in the image.

In particular, let $\tilde{\vel}_{i, {\rm T, pix}}\in\mathbb{R}^2$ be the velocity of the center pixel of the bounding box for observer $i$. The measurement of the bearing rate expressed in the camera frame can be obtained as
\begin{align}
\tilde{\h}_{i, {\rm c}} & = \P_{\tilde{\g}_{i, {\rm c}}}\mathbf{K}^{\rm c}_{i, {\rm pix}}\begin{bmatrix}
\tilde{\vel}_{i, {\rm T, pix}}(x)\\\tilde{\vel}_{i, {\rm T, pix}}(y)\\1
\end{bmatrix} = \h_{i, {\rm c}} + \mathbf{w}_{\h_{i, {\rm c}}},
\end{align}
where $\mathbf{w}_{\h_{i, {\rm c}}}$ is a noise, $\P_{\tilde{\g}_{i, {\rm c}}}=\I - \tilde{\g}_{i, {\rm c}}\tilde{\g}_{i, {\rm c}}^T$. The bearing rate in the camera frame can be converted to the world frame by taking the derivative of both sides of \eqref{eq_rel_cam2world_g} as
\begin{align}
\tilde{\h}_i = &~\P_{\g_{i}}(\dot{\R}_{i, {\rm c}}^{\rm w}\tilde{\g}_{i, {\rm c}} + \R_{i, {\rm c}}^{\rm w}\tilde{\h}_{i, {\rm c}})\nonumber\\
= & ~\h_i + \P_{\g_{i}}\underbrace{\dot{\R}_{i, {\rm c}}^{\rm w}\R_{i, {\rm c}}^{\rm w}\mathbf{w}_{\g_{i, {\rm c}}} + \R_{i, {\rm c}}^{\rm w}\mathbf{w}_{\h_{i, {\rm c}}}}_{\mathbf{w}_{\h_i}},\label{eq_noisy_gdot}
\end{align}
where $\dot{\R}_{i, {\rm c}}^{\rm w}$ is the time derivative of the rotational transformation $\R_{i, {\rm c}}^{\rm w}$. Based on \cite{zhao2016time}, $\dot{\R}_{i, {\rm c}}^{\rm w}$ can be obtained as $\dot{\R}_{i, {\rm c}}^{\rm w} = \R_{i, {\rm c}}^{\rm w} [\omega_{i, {\rm c}}]_{\times}$. The angular velocity $\omega_{i, {\rm c}}$ can be measured by the gyroscope sensor.

\subsection{System Architecture}

The system architecture for cooperative target motion estimation is shown in Fig.~\ref{fig_framework}.
The measurement equations corresponding to the vision measurements are presented in Sections~\ref{sub_g} and \ref{sub_h}.
The main contribution of this paper is to propose a novel estimator that can fully utilize the vision measurements. The details are given in Section~\ref{sec_STT_R}.

\section{State and Measurement Equations}

In this section, we establish the state and measurement equations.
The measurement equations indicate the relationship between the measurements and the target's state. They are used later when we design our estimator.

\subsection{State Equations}
The state vector of the target is $\x_k =[\p_{{\rm T}, k}^T,\vel_{{\rm T}, k}^T]^T\in\mathbb{R}^6$. If no information about the target's motion is available, it is common to model the target's motion as a discrete-time noise-driven double-integrator \cite{li2022three, ning2023comparison}:
 $\x_{k+1} = \A\x_{k} +\mathbf{B}\mathbf{w}_{k}.$
where $\A$ and $\mathbf{B}$ are
\begin{align}
\A
&=\begin{bmatrix}
\I & \Delta t\I \\
\zero&\I
\end{bmatrix}\in\mathbb{R}^{6\times6},
\quad \mathbf{B}
=\begin{bmatrix}
\frac{1}{2}\Delta t^2\I\\
\Delta t \I
\end{bmatrix}\in\mathbb{R}^{6\times 3}. 
\end{align}
Here, $\mathbf{w}_k\in\mathbb{R}^3$ is a zero-mean Gaussian noise with the covariance matrix as $\Sigma_{\mathbf{w}} = \sigma_{\mathbf{w}}^2\I$; $\Delta t$ is the sampling time, and $\I$ is the identity matrix.

\subsection{Measurement Equations}

Next, we derive the nonlinear measurement equations to make the vision measurements pseudo-linear to achieve better estimation stability.

First, the bearing measurement $\tilde{\g}_i$ can be expressed as
\begin{align}
\tilde{\g}_i =\R_{\epsilon_i}\g_i = \g_i + \underbrace{(\R_{\epsilon_i} - \I)\g_i}_{\boldsymbol{\mu}_i},
\label{eq_measuredBearingVector}
\end{align}
where $\R_{\epsilon_i}$ is a rotation matrix that perturbs $\g_i$ by a small random angle $\epsilon_i$. The rotation axis is a random unit vector orthogonal to $\g_i$.

It is clear that \eqref{eq_measuredBearingVector} is a nonlinear equation in the target's state.
Based on the derivation in \cite{zheng2023optimal}, we have
\begin{align}
\underbrace{\P_{\tilde{\g}_i}\p_i}_{\z_{\g_i}} = \underbrace{\begin{bmatrix}
\P_{\tilde{\g}_i} & \zero
\end{bmatrix}}_{\H_{\g_i}}\x + \underbrace{r_i\P_{\tilde{\g}_i}\boldsymbol{\mu}_i}_{\boldsymbol{\nu}_{\g_i}},
\label{eq_g_pseudo}
\end{align}
where $\z_{\g_i}$ is the pseudo-linear measurement.

Second, the measurement equation corresponding to the bearing rate measurement is derived as follows.
Substituting \eqref{eq_gdot} to \eqref{eq_noisy_gdot} yields
\begin{align}
\tilde{\h}_i = \frac{\P_{\g_i}(\vel_{{\rm T}}-\vel_i)}{r_i} + \P_{\g_i}\w_{\h_i},
\end{align}
which can be reorganized as
\begin{align}
\tilde{\h}_ir_i + \P_{\g_i}\vel_i = \P_{\g_i}\vel_{{\rm T}} + r_i\P_{\g_i}\w_{\h_i}.\label{eq_gdot_r_rel}
\end{align}
Since $r_i$ and $\P_{\g_i}$ are unknown, we next remove them from this equation. According to \eqref{eq_bearingAccurate}, we have $r_i =\g_i^T(\p_{\rm T} - \p_i)$. Substituting $\g_i=\tilde{\g}_i -\boldsymbol{\mu}_i$, according to \eqref{eq_measuredBearingVector}, into $r_i =\g_i^T(\p_{\rm T} - \p_i)$ gives
\begin{align}
r_i = \tilde{\g}_i^T(\p_{\rm T} - \p_i) - \underbrace{\boldsymbol{\mu}_i^T(\p_{\rm T} - \p_i)}_{\nu_{r_i}}.\label{eq_r_tilde_g}
\end{align}
Furthermore, the relationship between $\P_{\g_i}$ and $\P_{\tilde{\g}_i}$ is
\begin{align}
\P_{\g_i} = \P_{\tilde{\g}_i} + \underbrace{\tilde{\g}_i\tilde{\g}_i^T - \g_i\g_i^T}_{\boldsymbol{\Sigma}_{\P_i}}.\label{eq_rel_P_g}
\end{align}
Substituting \eqref{eq_r_tilde_g} and \eqref{eq_rel_P_g} into \eqref{eq_gdot_r_rel} yields
\begin{align}
& \tilde{\h}_i\tilde{\g}_i^T(\p_{\rm T} - \p_i) + \P_{\tilde{\g}_i}\vel_i \\
& = \P_{\tilde{\g}_i}\vel_{\rm T} + r_i\P_{\g_i}\w_{\h_i} + \tilde{\h}_i\nu_{r_i} + \boldsymbol{\Sigma}_{\P_i}(\vel_{\rm T} - \vel_i),
\end{align}
which can be reorganized as
\begin{align}
\underbrace{-\tilde{\h}_i\tilde{\g}_i^T\p_i + \P_{\tilde{\g}_i}\vel_i}_{\z_{\h_i}} & = \underbrace{\begin{bmatrix}
-\tilde{\h}_i\tilde{\g}_i^T& \P_{\tilde{\g}_i}
\end{bmatrix}}_{\H_{\h_i}}\x  + \boldsymbol{\nu}_{\h_i},\label{eq_rate_pseudo}
\end{align}
where $\boldsymbol{\nu}_{\h_i} = r_i\P_{\g_i}\w_{\h_i} + \tilde{\h}_i\nu_{r_i}+ \boldsymbol{\Sigma}_{\P_i}(\vel_{\rm T} - \vel_i)$.

Up to now, we have successfully established the pseudo-linear measurement equations \eqref{eq_g_pseudo} and \eqref{eq_rate_pseudo} for $\tilde{\g}_i$ and $\tilde{\h}_i$, respectively. 

\section{Cooperative Motion Estimator}\label{sec_STT_R}

In this section, we propose a new motion estimator called \emph{spatial-temporal triangulation with bearing rate} (STT-R).

\subsection{Objective Function}

The estimation process is formulated as a recursive least-squares problem.
The objective function should be well-designed to incorporate all available information, including visual measurements and the information exchanged between neighboring observers. Moreover, the objective function should balance estimation performances, including smoothness and convergence speed.

To that end, the objective function for observer $i$ at time $k$ is designed as
\begin{align}
&J\left(\hat{\x}_{i,k}\right) =\sum_{t=1}^{k}\lambda_t^{(k)}\left(c_1e^{\g}_{i,t}+c_2e^{\h}_{i,t} + e^{\rm cons}_{i,t}\right), \label{eq_J_x_k}
\end{align}
where $c_1$ and $c_2$ are positive weights, and $\lambda_t^{(k)}$ is a forgetting factor, which is an important parameter that will be discussed in detail later.
Here, $e^{\g}_{i,t}$, $e^{\h}_{i,t}$, and $e^{\rm cons}_{i,t}$ correspond to three types of \emph{inconsistency}: $e^{\g}_{i,t}$ corresponds to the inconsistency between current estimates and the bearing measurements; $e^{\h}_{i,t}$ corresponds to the inconsistency between current estimates and the bearing rate measurements; and $e^{\rm cons}_{i,t}$ corresponds to the inconsistency between estimates of different observers.
In particular, they are designed as
\begin{align}
e^{\g}_{i,t} & =\sum_{j\in(i\cup \N_{i,t})}\alpha_{ij, t}\left\|\z_{\g_{j, t}} - \H_{\g_{j, t}}\A^{t-k}\hat{\x}_{i,k}\right\|^2_{\R_{\g}}, \\
e^{\h}_{i,t} & = \sum_{j\in(i\cup \N_{i,t})} \beta_{ij, t}\left\|\z_{\h_{j, t}} - \H_{\h_{j, t}}\A^{t-k}\hat{\x}_{i,k}\right\|^2_{\R_{\h}}, \\
e^{\rm cons}_{i,t} & = \sum_{j\in(i\cup \N_{i,t})} \zeta_{ij, t}\left\|\hat{\x}_{j, t} - \A^{t-k}\hat{\x}_{i,k}\right\|^2,
\end{align}
where $\N_{i,t}$ is the set of the neighboring observers of observer $i$ at time $t$, $\alpha_{ij, t}\geq 0, \beta_{ij, t}\geq 0, \zeta_{ij, t}\geq 0$ are the weight values and $\zeta_{ij, t}$ satisfies $\sum_{j\in(i\cup\N_{i,t})}\zeta_{ij, t} = 1$.
Note that $\|\x\|^2_\mathbf{W}=\x^T\mathbf{W}\x$ for any vector $\x$ and weight matrix $\mathbf{W}$. The weight matrices $\R_{\g}$ and $\R_{\h}$ are selected as $\R_{\g} = \I/\sigma_{\nu_{\g}}^2$ and $\R_{\h} = \I/\sigma_{\nu_{\h}}^2$.
The forgetting factor is designed as \(\lambda_t^{(k)} = \gamma_2^{k-t}/\gamma_1^{k-t+1}\),
where $\gamma_1$ and $\gamma_2$ are positive constants.

\subsection{STT-R Estimator}

The recursive STT-R algorithm is given as follows. It consists of three steps: \emph{prediction}, \emph{innovation}, and \emph{correction}.
\begin{align}	
\intertext{\textbf{Prediction:}}
\hat{\x}^{-}_{i,k} & = \A\hat{\x}_{i,k-1}, \label{eq_x_pred}\\
\hat{\M}^{-}_{i,k} & = \frac{1}{\gamma_1}\left(\A\hat{\M}_{i,k-1}\A^T\right)^{-1}, \label{eq_M_pred}\\
\intertext{\textbf{Innovation:}}
\mathbf{e}^{\g}_{i,k} & = c_1 \sum_{j\in(i\cup\N_{i,t})} \alpha_{ij, k} \H_{\g_{j, k}}^T\R_{\g}\left(\z_{\g_{j, k}} - \H_{\g_{j, k}}\hat{\x}_{i,k}^{-} \right), \label{eq_error_z_g}\\
\mathbf{e}^{\h}_{i,k} & =c_2 \sum_{j\in(i\cup\N_{i,t})} \beta_{ij, k} \H_{\h_{j, k}}^T\R_{\h}\left(\z_{\h_{j, k}} - \H_{\h_{j, k}}\hat{\x}_{i,k}^{-} \right), \label{eq_error_z_gdot}\\	
\mathbf{e}^{\text{cons}}_{i,k} & = \sum_{j\in(i\cup\N_{i,k})} \zeta_{ij, k} \left(\hat{\x}_{j, k}^{-} - \hat{\x}_{i,k}^{-} \right), \label{eq_error_consensus}\\	
\S_{i,k} & =c_1\sum_{j\in(i\cup\N_{i,t})}\alpha_{ij, k}\H_{\g_{j, k}}^T\R_{\g}\H_{\g_{j, k}} \nonumber\\
& \qquad + c_2\sum_{j\in(i\cup\N_{i,t})}\beta_{ij, k}\H_{\h_{j, k}}^T\R_{\h}\H_{\h_{j, k}}+ \I , \label{eq_covariance_S}\\
\intertext{\textbf{Correction:}}
\hat{\M}_{i,k} & = (\gamma_2\hat{\M}_{i,k}^{-}+ \S_{i,k})^{-1}, \label{eq_M_correction}\\
\hat{\x}_{i,k} & = \hat{\x}_{i,k}^{-} + \hat{\M}_{i,k} (\mathbf{e}^{\g}_{i,k} + \mathbf{e}^{\h}_{i,k}+\mathbf{e}^{\text{cons}}_{i,k}).\label{eq_x_correction}
\end{align}

\emph{Prediction:}
$\hat{\x}_{i,k-1}$ and $\hat{\M}_{i,k-1}$ are the optimal estimates for the previous time $(k-1)$. They are predicted with the state transition model given in \eqref{eq_x_pred} and \eqref{eq_M_pred} for the next time step update.

\emph{Innovation:} Three different errors are calculated with the new arrival information given in \eqref{eq_error_z_g}, \eqref{eq_error_z_gdot}, and \eqref{eq_error_consensus}, respectively. $\mathbf{e}^{\g}_{i,k}$ is the error in the bearing measurements $\g$. $\mathbf{e}^{\h}_{i,k}$ is the error in the bearing rate measurements $\h$. $\mathbf{e}^{\rm cons}_{i,k}$ is the estimated states' inconsistent error between neighboring observers. $\S_{i,k}$ in \eqref{eq_covariance_S} is the sum weight matrix of these three errors.

\emph{Correction:} $\hat{\M}_{i,k}^{-}$ is rectified by the innovative information $\S_{i,k}$ as shown in \eqref{eq_M_correction}. Subsequently, in \eqref{eq_x_correction}, the predicted target's state $\hat{\x}_{i,k}^{-}$ is updated with the three innovation errors.

\section{Observability Analysis}\label{sec_obs}

Although the STT-R estimator incorporates an additional rate measurement, it is not immediately clear how this measurement can enhance the system's observability. 

The proof relies on the following lemma about orthogonal projection matrices. This lemma is also used in many other analyses in this paper.

\begin{lemma}[\cite{zheng2023optimal}]\label{lemma_rank_P_1_P_2}
For any unit bearing vectors $\g_1, \g_2\in\mathbb{R}^3$, we have ${\rm rank}(\P_{\g_1}+\P_{\g_2})=3$ if and only if $\g_1$ is not parallel to $\g_2$.
\end{lemma}

We can obtain the observability matrix as
\begin{align}
\mathbf{Q}_{i,k} =
\begin{bmatrix}
w_{i1,k}\H_{\g_{1, k}}\A^k\\
w_{i1,k}\H_{\h_{1, k}}\A^k\\
\vdots\\
w_{in,k}\H_{\g_{n, k}}\A^k\\
w_{in,k}\H_{\h_{n, k}}\A^k
\end{bmatrix}.\label{eq_obv_matrix}
\end{align}
Here, $w_{ij,k}=1$ if observer $i$ can obtain the information from observer $j$ at time $k$. Otherwise, $w_{ij,k}=0$.
We always have $w_{ii,k}=1$ for any $k$.

\begin{figure*}
    \centering
    \includegraphics[width=\textwidth]{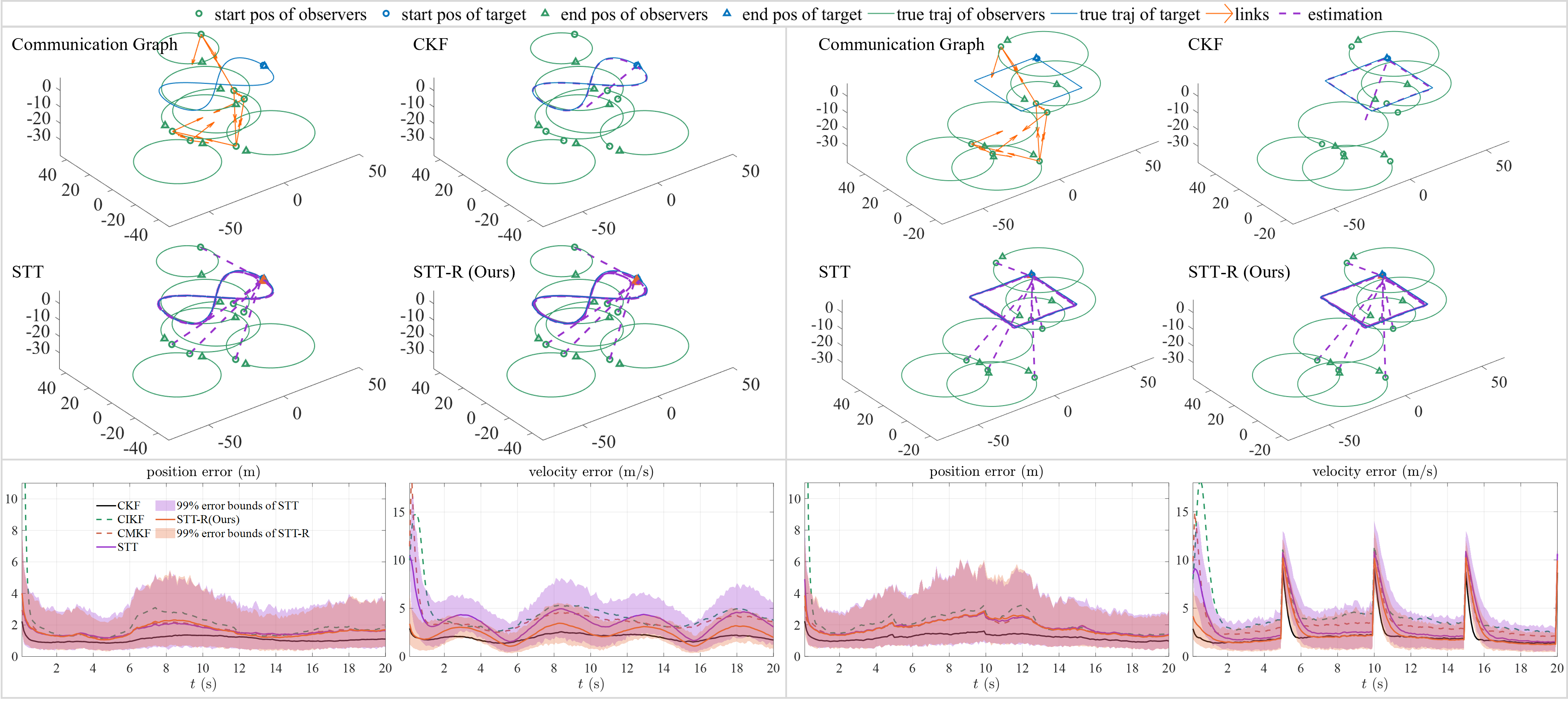}
    \caption{The simulation results of CKF, CIKF, CMKF, STT, and STT-R in 8-shape motion and square-shape motion, respectively.}\label{fig_sim}
\end{figure*}

\begin{theorem}[Multi-observer observability condition]\label{theo_multi_obs}
The observability matrix $\mathbf{Q}_{i,k}$ is of full column rank if and only if at least one neighbor of observer $i$ has a bearing vector that is non-parallel to $\g_{i,k}$.
\end{theorem}
\begin{proof}
\emph{Sufficiency:}
Suppose observer $j$'s bearing vector $\g_{j,k}$ is non-parallel to $\g_{i,k}$.
Define
\begin{align}
\mathbf{Q}_{ij,k} = \begin{bmatrix}
\P_{\g_{i,t}} & k\Delta t\P_{\g_{i,t}}\\
-\h_{i,t}\g_{i,t}^T & \P_{\g_{i,t}} - k\Delta\h_{i,t}\g_{i,t}^T  \\
\P_{\g_{j, t}} &k\Delta t\P_{\g_{j, t}}\\
-\h_{j, t}\g_{j, t}^T & \P_{\g_{j, t}} - k\Delta t\h_{j, t}\g_{j, t}^T
\end{bmatrix}.
\end{align}		
Since $\mathbf{Q}_{ij,k}$ is a sub-matrix of $\mathbf{Q}_{i,k}$, we have ${\rm rank}(\mathbf{Q}_{i,k}) \ge {\rm rank}(\mathbf{Q}_{ij,k})$.
After a series of elementary transformations of $\mathbf{Q}_{ij,k}$, we have
\begin{align*}
\mathbf{Q}_{ij,k} \rightarrow  \begin{bmatrix}
\P_{\g_{i,t}} & \zero\\
\P_{\g_{j, t}} & \zero\\
-\h_{i,t}\g_{i,t}^T & \P_{\g_{i,t}}  \\
-\h_{j, t}\g_{j, t}^T & \P_{\g_{j, t}}
\end{bmatrix}\rightarrow \begin{bmatrix}
\I & \zero\\
\zero &\zero\\
\zero & \P_{\g_{i,t}} \\
\zero & \P_{\g_{j, t}}
\end{bmatrix}\triangleq \mathbf{C}.
\end{align*}
According to the properties of the matrix rank, we have
\begin{align*}
{\rm rank}(\mathbf{C}) & = {\rm rank}(\mathbf{C}^{T}\mathbf{C}) = {\rm rank}\left(\begin{bmatrix}
\I & \zero\\
\zero & \P_{\g_{i,t}}+ \P_{\g_{j, t}}
\end{bmatrix}\right).
\end{align*}
When deriving the above equation, we used the property that $\P_{\g_{i,t}}^T = \P_{\g_{i,t}} = \P_{\g_{i,t}}^2$.
Since $\g_{i,t}$ and $\g_{j,t}$ are non-parallel, we know that $\P_{\g_{i,t}}+ \P_{\g_{j, t}}$ is of full rank according to Lemma~\ref{lemma_rank_P_1_P_2}.
Therefore, ${\rm rank }(\mathbf{C})=6$ and hence ${\rm rank}(\mathbf{Q}_{i,k})\ge {\rm rank}(\mathbf{Q}_{ij,k})=6$. Since $\mathbf{Q}_{i,k}$ has only 6 columns, we know ${\rm rank}(\mathbf{Q}_{i,k})=6$.

\emph{Necessity:}
Suppose that all the neighbors' bearing vectors of observer $i$ are parallel to $\g_{i,k}$. Then after a series of elementary operations, $\mathbf{Q}_{i,k}$ can be expressed as
\begin{align*}
\mathbf{Q}_{i,k} \rightarrow \begin{bmatrix}
w_{i1,k}\P_{\g_{1,k}} & \zero\\
\vdots & \vdots\\
w_{in,k}\P_{\g_{n,k}} & \zero\\
-w_{i1,k}\h_{1, k}\g^T_{1,k} & w_{i1,k}\P_{\g_{1,k}}\\
\vdots & \vdots\\
-w_{in,k}\h_{n, k}\g^T_{n,k} & w_{in,k}\P_{\g_{n,k}}
\end{bmatrix}.
\end{align*}
According to Lemma~\ref{lemma_rank_P_1_P_2}, the rank of the last three columns of $\mathbf{Q}_{i,k}$ is 2 since all available bearing vectors ($w_{ij,k}=1$) are parallel.
As a result, $\mathbf{Q}_{i,k}$ is rank deficient in this case.
\end{proof}

\section{Simulation Verification}\label{sec_sim}

\subsection{Simulation Setup}

\begin{figure*}
\centering
\includegraphics[width=\linewidth]{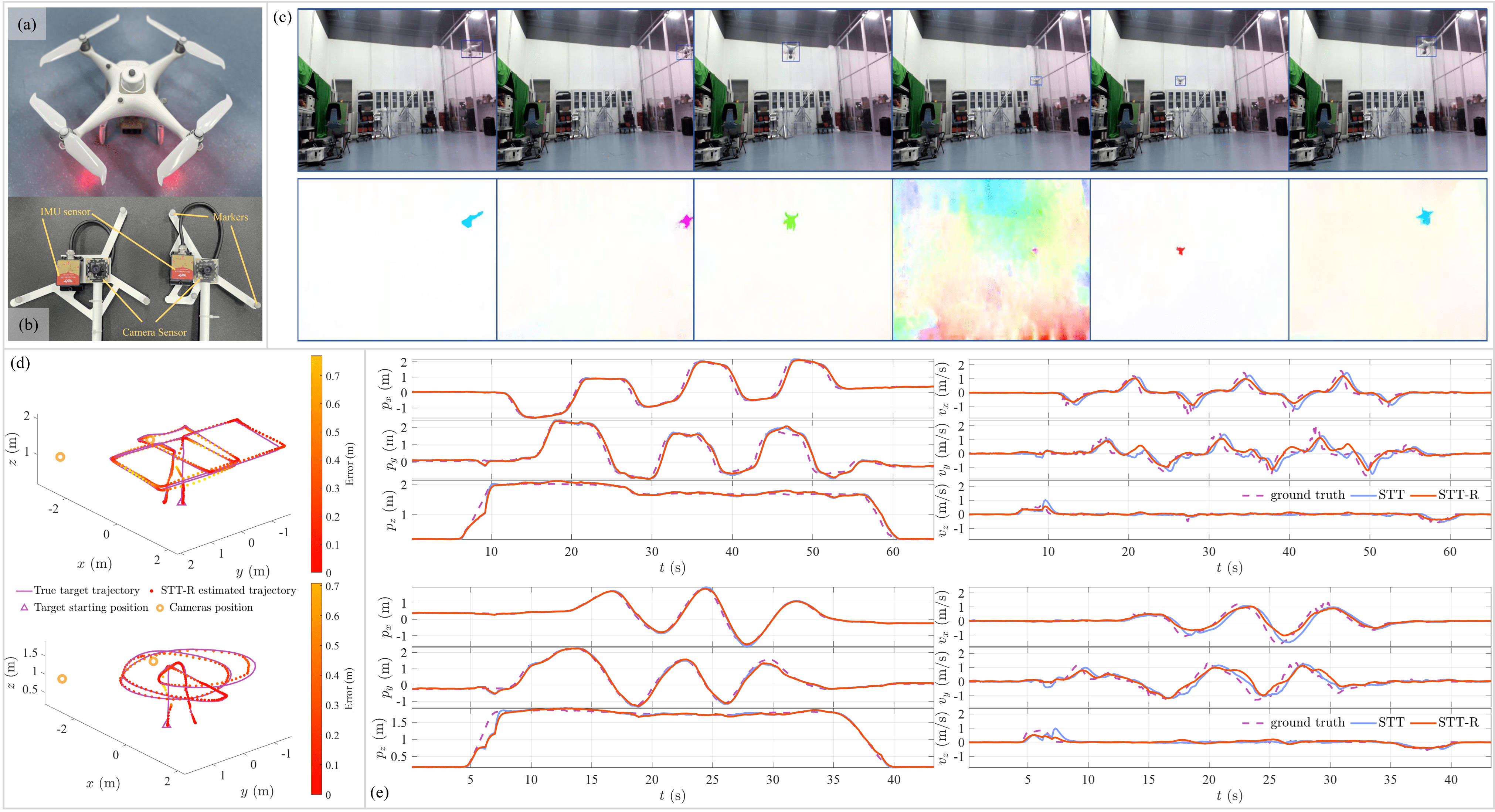}
\caption{(a) is the target MAV (DJI Phantom 4). (b) are the two observation devices. Each observer device consists of a camera and an IMU sensor. (c) The target detection results (upper images) and the optical flow detection results (lower images). The color of each pixel in the optical flow detection results indicates the velocity of that pixel in the image frame. (d) shows the 3D Estimated trajectories in two scenarios. (e) shows the estimated positions and velocities in two scenarios, respectively.}\label{fig_exp}
\end{figure*}

Six observers are randomly generated in the space of $80\times 80\times 40~{\rm m}^3$. To make the simulation scenario more complicated, we let each observer undergo circular motion with different angular velocities. Every observer can exchange information with the three nearest neighbors ($|\mathcal{N}_i|=3$) at every time step.

The performance of the STT-R algorithm is compared with the central Kalman filter (CKF), consensus on measurements Kalman filter (CMKF) \cite{olfati2009kalman}, consensus on information Kalman filter (CIKF) \cite{olfati2005consensus}, and the original STT algorithm \cite{zheng2023optimal}. The parameters of the three algorithms are given in Table~\ref{tab_params}.
In particular, the performance of CKF is influenced by three matrices: $\R_{\g}$, $\R_{\h}$, and $\mathbf{Q}$. Here, we select $\R_\g = r_\g^{-1}\sigma_\g^{-2}\I$, $\R_{\h} = r_{\h}^{-1}\sigma_{\h}^{-2}\I$, and $\mathbf{Q} =q[dt^3/3,dt^2/2;dt^2/2,dt]\otimes \I_3$. All the parameters in Table~\ref{tab_params} have been optimized by using a genetic algorithm \cite{conn1991globally}.

\begin{table}[t]
\caption{Optimized parameters of algorithms.}\label{tab_params}
\centering
\begin{tabular}{c|ccccc}
\hline
Parameters &  CKF  & CIKF & CMKF &  STT  & STT-R  \\ \hline
$r_\g$ or $\alpha$    & 0.984 & 1.135 & 3.543 & 1.328 & 1.328  \\
$r_{\h}$ or $\beta$ & 6.354  & $-$ & $-$ & $-$   & 1.442 \\
$q$     & 8.96e-3 & 1.48e-2  & 6.96e-2 & $-$ & $-$ \\
$\gamma_1$  & $-$ & $-$ & $-$ & 8.1481 & 8.1481\\
$\gamma_2$  & $-$ & $-$ & $-$ &  6   & 6\\
c ($c_1$)  & $-$ & $-$ & $-$ & 0.354 & 0.165 \\
$c_2$  & $-$ & $-$ & $-$ & $-$ & 0.032 \\
$\zeta$ & $-$ & 0.25 & 0.25 & 0.25 & 0.25 \\
\hline
\end{tabular}
\end{table}

\emph{Setup:}
Two different motions are assigned to the target: an 8-shaped motion and a square-shaped motion. In the 8-shaped motion, the varying velocity is $\vel_{{\rm T}} = 10[-\sin(t\pi/10), \cos(t\pi/5), 0]^T$. In the square-shaped motion, the target's velocity is $8~{\rm m/s}$, with a 90-degree direction rotation every 5 seconds.
The initial estimated target position of CKF is the central position of all the observers $\hat{\x}_0=[\sum^n_{i=1}\p_{i, 0}^T, \zero_{3\times1}^T]^T$. The initial estimate of every observer is selected as $\hat{\x}_{i, 0}=[\p_{i, 0}^T, \zero_{3\times1}^T]^T$ in the CIKF, CMKF, STT and STT-R algorithms.
The standard deviations of different noises are given as $\sigma_{\g_w}=5.7~\text{degree}$, $\sigma_{\h_c}=4.6~\text{degree}/{\rm s}$, and $\sigma_{\omega_c}=1.1~\text{degree}/{\rm s}$, respectively.

\emph{Result:}
Fig.~\ref{fig_sim} shows the estimation results of the CKF, CIKF, CMKF, STT, and STT-R, respectively.
We have the following observations. First, STT-R's velocity estimation is significantly better than that of CIKF, CMKF, and STT. That is because it properly utilizes additional bearing rate measurements. Second, STT-R's position estimation is comparable to that of CIKF, CMKF, and STT.
That is partially because, from an information theoretical perspective, the information used for position estimation by all these algorithms is the same as the bearing measurements.

\section{Real-World Experiments}\label{sec_exp}

In this section, the performance of the proposed STT-R algorithm is evaluated in real-world experiments. We conduct two indoor experiments with two observers tracking a flying MAV.
The experiment setup is depicted in Fig.~\ref{fig_env}. The target MAV and the observers are shown in Fig.~\ref{fig_exp}~(a) and (b), respectively. The ground-truth states of the two observers and the target MAV are obtained by the indoor motion capture system.

To detect the MAV target, we train a YOLOv5s detector based on a dataset of more than 4,000 images of the target MAV in the experimental environment. Starting from pre-trained weights of YOLOv5s on MS-COCO, we train the detector for 200 epochs with a batch size of 32. The camera's intrinsic parameters are calibrated beforehand. Readers can, of course, use other detectors as well.
To detect the optical flow, we employ the off-the-shelf optical flow model RAFT \cite{2020RAFT} trained on synthetic datasets \cite{raft-things} to extract the target's velocity.
Samples of target detection and optical flow detection results are shown in Fig.~\ref{fig_exp}~(e).
We tested the computational efficiency of the detection, optical flow, and estimation algorithms on an Nvidia Jetson Xavier NX embedded computer.

Two experiments were conducted. In the first one, the target MAV makes a rapid directional change, resembling a square trajectory. In the second one, the target MAV flies randomly. The experimental results are shown in Fig.~\ref{fig_exp}~(d) and (e).

The experimental results demonstrate that STT-R achieves better estimation accuracy. The lag in the velocity estimation is also reduced.

Compared with the STT algorithm, the introduction of bearing rate measurements improves the tracking performance in the target's state estimation, especially in the target's velocity tracking. When tracking a maneuvering target, the STT-R algorithm reduces the lag in velocity tracking and improves the tracking accuracy.

\section{Conclusion}\label{sec_conclusion}

In this paper, we propose a new motion estimator called STT-R. It incorporates the bearing rate measurements that have not been well utilized in the past and can hence significantly improve the system observability. The necessary and sufficient observability conditions show that the cooperative system is observable for any single time step. This is a significant observability enhancement compared to the bearing-only case, where multiple time steps are required to ensure observability. Numerical simulations and real-world experiments are presented to verify the effectiveness of the proposed estimator.

\bibliography{ zcl_references}

\begin{thebibliography}{10}

\bibitem{vrba2020marker}
M.~Vrba and M.~Saska, ``Marker-less micro aerial vehicle detection and
  localization using convolutional neural networks,'' {\em IEEE Robotics and
  Automation Letters}, vol.~5, no.~2, pp.~2459--2466, 2020.

\bibitem{li2022three}
J.~Li, Z.~Ning, S.~He, C.-H. Lee, and S.~Zhao, ``Three-dimensional bearing-only
  target following via observability-enhanced helical guidance,'' {\em IEEE
  Transactions on Robotics}, vol.~39, no.~2, pp.~1509--1526, 2023.

\bibitem{saini2019markerless}
N.~Saini, E.~Price, R.~Tallamraju, R.~Enficiaud, R.~Ludwig, I.~Martinovic,
  A.~Ahmad, and M.~Black, ``Markerless outdoor human motion capture using
  multiple autonomous micro aerial vehicles,'' in {\em 2019 IEEE/CVF
  International Conference on Computer Vision (ICCV)}, pp.~823--832, 2019.

\bibitem{brighton2019hawks}
C.~H. Brighton and G.~K. Taylor, ``Hawks steer attacks using a guidance system
  tuned for close pursuit of erratically manoeuvring targets,'' {\em Nature
  communications}, vol.~10, no.~1, p.~2462, 2019.

\bibitem{griffin2021depth}
B.~A. Griffin and J.~J. Corso, ``Depth from camera motion and object
  detection,'' in {\em 2021 IEEE/CVF Conference on Computer Vision and Pattern
  Recognition (CVPR)}, pp.~1397--1406, 2021.

\bibitem{su2022bearing}
H.~Su, C.~Chen, Z.~Yang, S.~Zhu, and X.~Guan, ``Bearing-based formation
  tracking control with time-varying velocity estimation,'' {\em IEEE
  Transactions on Cybernetics}, vol.~53, no.~6, pp.~3961--3973, 2023.

\bibitem{sharma2011graph}
R.~Sharma, R.~W. Beard, C.~N. Taylor, and S.~Quebe, ``Graph-based observability
  analysis of bearing-only cooperative localization,'' {\em IEEE Transactions
  on Robotics}, vol.~28, no.~2, pp.~522--529, 2012.

\bibitem{ning2024bearingangle}
Z.~Ning, Y.~Zhang, J.~Li, Z.~Chen, and S.~Zhao, ``A bearing-angle approach for
  unknown target motion analysis based on visual measurements,'' {\em The
  International Journal of Robotics Research,(arXiv:2401.17117)}, vol.~43,
  no.~8, pp.~1228--1249, 2024.

\bibitem{flayac2023nonuniform}
E.~Flayac and I.~Shames, ``Nonuniform observability for moving horizon
  estimation and stability with respect to additive perturbation,'' {\em SIAM
  Journal on Control and Optimization}, vol.~61, no.~5, pp.~3018--3050, 2023.

\bibitem{sharma2013bearing}
R.~Sharma, S.~Quebe, R.~W. Beard, and C.~N. Taylor, ``Bearing-only cooperative
  localization,'' {\em Journal of Intelligent \& Robotic Systems}, vol.~72,
  no.~3, pp.~429--440, 2013.

\bibitem{zhao2014optimal}
S.~Zhao, B.~M. Chen, and T.~H. Lee, ``Optimal deployment of mobile sensors for
  target tracking in 2d and 3d spaces,'' {\em IEEE/CAA Journal of Automatica
  Sinica}, vol.~1, no.~1, pp.~24--30, 2014.

\bibitem{schiano2018dynamic}
F.~Schiano and R.~Tron, ``The dynamic bearing observability matrix nonlinear
  observability and estimation for multi-agent systems,'' in {\em 2018 IEEE
  International Conference on Robotics and Automation (ICRA)}, pp.~3669--3676,
  2018.

\bibitem{zheng2023optimal}
C.~Zheng, Y.~Mi, H.~Guo, H.~Chen, F.~Chen, J.~Jia, Z.~Lin, and S.~Zhao,
  ``Optimal spatial-temporal triangulation for bearing-only cooperative motion
  estimation,'' {\em arXiv preprint arXiv:2310.15846}, 2023.

\bibitem{shen2013vision}
S.~Shen, Y.~Mulgaonkar, N.~Michael, and V.~Kumar, ``Vision-based state
  estimation and trajectory control towards high-speed flight with a
  quadrotor.,'' in {\em Robotics: science and systems}, vol.~1, p.~32, Berlin,
  Germany, 2013.

\bibitem{gardner2022pose}
M.~Gardner and Y.-B. Jia, ``Pose and motion estimation of free-flying objects:
  aerodynamics, constrained filtering, and graph-based feature tracking,'' {\em
  IEEE Transactions on Robotics}, vol.~38, no.~5, pp.~3187--3202, 2022.

\bibitem{riaz2022state}
S.~Riaz and A.-H.~I. Mourad, ``State estimation of fixed wing micro air vehicle
  through cascaded discrete time extended kalman filter (ekf) schemes,'' in
  {\em 2022 Advances in Science and Engineering Technology International
  Conferences (ASET)}, pp.~1--6, IEEE, 2022.

\bibitem{2020RAFT}
Z.~Teed and J.~Deng, ``{RAFT}: Recurrent all-pairs field transforms for optical
  flow,'' in {\em Proceedings of the 16th European Conference on Computer
  Vision (ECCV)}, pp.~402--419, Springer, 2020.

\bibitem{zhao2019bearing}
S.~Zhao and D.~Zelazo, ``Bearing rigidity theory and its applications for
  control and estimation of network systems: Life beyond distance rigidity,''
  {\em IEEE Control Systems Magazine}, vol.~39, no.~2, pp.~66--83, 2019.

\bibitem{mateos2012distributed}
G.~Mateos and G.~B. Giannakis, ``Distributed recursive least-squares: Stability
  and performance analysis,'' {\em IEEE Transactions on Signal Processing},
  vol.~60, no.~7, pp.~3740--3754, 2012.

\bibitem{redmon2018yolov3}
J.~Redmon and A.~Farhadi, ``Yolov3: An incremental improvement,'' {\em arXiv
  preprint arXiv:1804.02767}, 2018.

\bibitem{bochkovskiy2020yolov4}
A.~Bochkovskiy, C.-Y. Wang, and H.-Y.~M. Liao, ``Yolov4: Optimal speed and
  accuracy of object detection,'' {\em arXiv preprint arXiv:2004.10934}, 2020.

\bibitem{2021Air}
Y.~Zheng, Z.~Chen, D.~Lv, Z.~Li, and S.~Zhao, ``Air-to-air visual detection of
  micro-{UAV}s: An experimental evaluation of deep learning,'' {\em IEEE
  Robotics and Automation Letters}, vol.~6, no.~2, pp.~1020--1027, 2021.

\bibitem{wu2020using}
D.~Wu, S.~Lv, M.~Jiang, and H.~Song, ``Using channel pruning-based yolo v4 deep
  learning algorithm for the real-time and accurate detection of apple flowers
  in natural environments,'' {\em Computers and Electronics in Agriculture},
  vol.~178, p.~105742, 2020.

\bibitem{isaac2021unmanned}
B.~K. Isaac-Medina, M.~Poyser, D.~Organisciak, C.~G. Willcocks, T.~P. Breckon,
  and H.~P. Shum, ``Unmanned aerial vehicle visual detection and tracking using
  deep neural networks: A performance benchmark,'' in {\em Proceedings of the
  IEEE/CVF International Conference on Computer Vision}, pp.~1223--1232, 2021.

\bibitem{raft-things}
N.~Mayer, E.~Ilg, P.~Hausser, P.~Fischer, D.~Cremers, A.~Dosovitskiy, and
  T.~Brox, ``A large dataset to train convolutional networks for disparity,
  optical flow, and scene flow estimation,'' in {\em Proceedings of the IEEE
  Conference on Computer Vision and Pattern Recognition}, pp.~4040--4048, 2016.

\bibitem{zhao2016time}
S.~Zhao, ``Time derivative of rotation matrices: A tutorial,'' {\em arXiv
  preprint arXiv:1609.06088}, 2016.

\bibitem{ning2023comparison}
Z.~Ning, Y.~Zhang, and S.~Zhao, ``Comparison of different pseudo-linear
  estimators for vision-based target motion estimation,'' {\em Control Theory
  and Technology}, vol.~21, no.~3, pp.~448--457, 2023.

\bibitem{olfati2009kalman}
R.~Olfati-Saber, ``Kalman-consensus filter: Optimality, stability, and
  performance,'' in {\em Proceedings of the 48h IEEE Conference on Decision and
  Control (CDC) held jointly with 2009 28th Chinese Control Conference},
  pp.~7036--7042, IEEE, 2009.

\bibitem{olfati2005consensus}
R.~Olfati-Saber and J.~S. Shamma, ``Consensus filters for sensor networks and
  distributed sensor fusion,'' in {\em Proceedings of the 44th IEEE Conference
  on Decision and Control}, pp.~6698--6703, IEEE, 2005.

\bibitem{conn1991globally}
A.~R. Conn, N.~I. Gould, and P.~Toint, ``A globally convergent augmented
  lagrangian algorithm for optimization with general constraints and simple
  bounds,'' {\em SIAM Journal on Numerical Analysis}, vol.~28, no.~2,
  pp.~545--572, 1991.

\end{thebibliography}
\bibliographystyle{ieeetr}

\end{document}